\documentclass{article}

\usepackage[utf8]{inputenc}
\usepackage{hyperref}
\usepackage[centertags]{amsmath}
\usepackage{amsfonts}
\usepackage{amssymb}
\usepackage{amsthm}
\usepackage{pdfpages}
\usepackage{color}
\usepackage{float}

\newcommand{\abs}[1]{\left\vert#1\right\vert}

\usepackage[linesnumbered,ruled]{algorithm2e}

\newtheorem{theorem}{Theorem}
\newtheorem{lemma}{Lemma}
\newtheorem{proposition}{Proposition}
\newtheorem{cor}{Corollary}
\newtheorem{definition}{Definition}

\newtheorem{remark}{Remark}

\usepackage{graphicx} % Required for inserting images

\usepackage{tikz}

\usetikzlibrary{shapes.geometric, arrows}

\tikzstyle{lemma} = [rectangle, 
minimum width=2cm, 
minimum height=1cm, 
text centered, 
text width=3cm, 
draw=black]

\tikzstyle{lemma2} = [rectangle, 
minimum width=1cm, 
minimum height=1cm, 
text centered, 
text width=2cm, 
draw=black]
\tikzstyle{arrow} = [thick,->,>=stealth]

\title{First and Second Order Approximations to Stochastic Gradient Descent Methods with Momentum Terms}
\author{Eric Lu \\ \href{mailto:lu.eric96@gmail.com}{lu.eric96@gmail.com}}

\date{April 2025}

\begin{document}

\maketitle
\begin{abstract}
    Stochastic Gradient Descent (SGD) methods see many uses in optimization problems. Modifications to the algorithm, such as momentum-based SGD methods have been known to produce better results in certain cases. Much of this, however, is due to empirical information rather than rigorous proof. While the dynamics of gradient descent methods can be studied through continuous approximations, existing works only cover scenarios with constant learning rates or SGD without momentum terms. We present approximation results under weak assumptions for SGD that allow learning rates and momentum parameters to vary with respect to time. 
\end{abstract}
\section{Introduction}
In this paper, we will describe an approach to approximating stochastic gradient descent methods by continuous processes. Such approximations have been commonly studied, such as in \cite{li2017stochasticmodifiedequationsadaptive}, in order to study the dynamics of discrete processes through stochastic calculus. The most basic form of a gradient descent method can be described as follows. To minimize an objective function of the form
$$
\frac{1}{\abs{\Gamma}} \sum_{i \in \Gamma} f_i ,
$$
with the family of functions $f_{i}:\mathbb{R}^d \rightarrow \mathbb{R}$, we define the iterative process of Stochastic Gradient Descent (SGD) by
\begin{equation}\label{SGDGeneralform}
\chi_{n+1} = \chi_n - \eta \nabla f_{\gamma(n)} (\chi_n) \text{,}
\end{equation}
where $(\gamma(n))_{n \in \mathbb{N}_0}$ is a sequence of i.i.d random variables taking values in $\Gamma$, and $\eta$ is a constant, known as the \textit{learning rate}. A common way to achieve faster convergence is to use so-called momentum-based gradient descent methods. For instance, given a sequence of \textit{momentum parameters} $\zeta_n$, the Heavy Ball Method has iterates
$$
\chi_{n+1} = \chi_n - \eta \nabla f_{\gamma(n)} (\chi_n) + \zeta_n ( \chi_n - \chi_{n-1} ).
$$ 
As proven in literature such as \cite{garrigos2024handbookconvergencetheoremsstochastic}, momentum-based SGD in a convex setting achieves a convergence rate of $\mathcal{O}(1/k^2)$, an improvement over SGD's rate of $\mathcal{O}(1/k)$. While in this paper we allow the momentum parameter to follow a schedule, it is typical to set momentum to a constant, around $0.9$ (see \cite{liu2020improvedanalysisstochasticgradient}).
Another technique for optimizing SGD, as considered in \cite{Bottou1991StochasticGL}, is to use non-constant learning rates by instead defining a sequence $(\eta_n)_{n \in \mathbb{N}_0}$ called a \textit{learning rate schedule}. Diffusion approximations to momentum-based SGD have been explored in \cite{liu2024odemethodstochasticapproximation} \nocite{liu2024odemethodstochasticapproximation} with constant small learning rates. Furthermore, weak approximation results to classic SGD with decreasing learning rates are proven in \cite{ankirchner2021stochasticapproximation}.     Approximation to SGD is covered in Theorem \ref{sgdodeapproxmain}, the result of which we extend the to learning rate schedules where entries $\eta_n$ take the form of diagonal matrices
$$ \eta_n = \begin{pmatrix}
\eta^1_n & & & \\
& \eta^2_n & & \\
& & \ddots &  \\
 & &  & \eta^d_n 
\end{pmatrix} .
$$
By studying the discrete process $\mathring{\chi} := (\chi_{n+1}, \chi_{n})$, we develop a diffusion approximation to momentum-based SGD. This technique builds upon previous results to establish sufficient growth and smoothness conditions under which we can approximate momentum-based SGD with decreasing learning rates by continuous processes. In addition to proving weak convergence, we also quantify the error expansion of this approximation relative to the maximal learning rate.\\~\\
Section \ref{sectionprelim} will cover basic results from stochastic calculus. This paper's main result will be stated for both Ordinary Differential Equation and Stochastic Differential Equation (SDE) approximations in Section \ref{sectionmaintheorem}, with the two cases proven separately in the subsequent Sections \ref{sectionodeproof} and \ref{sectionsdeproof}.
\section{Preliminaries}\label{sectionprelim}
This section contains notation and the forms of basic lemmas that will be necessary for the proof of the main theorem, as well as key assumptions under which we will operate. 
\subsection{Spaces}
%\cite{alma990016311390107876}
A function $f \in C^\ell ( [0,T] \times \mathbb{R}^d)$ is said to be in $\text{Lip}^\ell$ if there exists a constant $L$ such that for any $x_1,x_2 \in \mathbb{R}^d$ and any multi index $\alpha = \{ \alpha_1,...,\alpha_d \}$ of nonnegative integers with $\abs{\alpha} := \alpha_1 + ... + \alpha_d \leq \ell$,
$$
\abs{\partial_{\alpha} f(x_1) - \partial_{\alpha} f(x_2)} \leq L \abs{x_1 - x_2} .
$$
For $\kappa \in \mathbb{N}$, and $D \subset \mathbb{R}^d$, we let $G_\kappa(D)$ be the space of functions $g: \mathbb{R}^d \rightarrow \mathbb{R}$ for which there exists a constant $C > 0$ such that
$$
\abs{g(x)} \leq C ( 1 + \|x \|^\kappa ), \quad \forall x \in D \text{.}
$$
We define the norm $\| \cdot \|_{G_\kappa}$ as
$$
\| g \|_{G_\kappa} := \inf_{C} \{ \abs{g(x)} \leq C(1 + \|x\|^\kappa); \quad \forall x \in D\}.
$$
Furthermore, let the space $G^\ell(D)$ contain all functions $g$ such that $g \in C^\ell (D)$ and every partial derivative of $g$, up to order $\ell$, is in $G_\kappa(D)$ for some $\kappa > 1$. Now let $(\Omega,\mathcal{F}_\Omega,\mathbb{P})$ with $\Omega \subset \mathbb{R}^d$ be a complete probability space. The vector space $L_p (\Omega)$ consists random variables with finite $p$-th absolute moments. That is, for $\xi \in L_p(\Omega)$ we have
$$
\| \xi \|_p := (\mathbb{E}\abs{\xi}^p)^{1/p} < \infty.
$$
We define a stochastic process $X_t$ as a map from $[0,T] \rightarrow L_p(\Omega)$ $t \in [0,T]$, $X_t \in L_p  (\Omega)$. Let $I$ be a set and $X = (X_t^i)_{i \in I,t \geq 0}$ be an $I$-indexed family of continuous-time stochastic processes. Given $p \in [1,\infty)$, we define the norms
\begin{align*}
    \| X \|_{p,t} &= \sup_{i \in I} \biggl( \mathbb{E} \int_0^t \abs{X_s^i}^p \, ds \biggr)^{1/p}, \\
    \| X^* \|_{p,t} &= \sup_{i \in I} \biggl( \mathbb{E} \sup_{s \in [0,t]} \abs{X_s^i}^p \biggr)^{1/p},
\end{align*}
and similarly define for  a discrete-time  stochastic process $X$ the norm
\begin{align*}
    \| X^* \|_{p,n} &= \sup_{i \in I} \biggl( \mathbb{E} \sup_{j \in [0,n]} \abs{X_j^i}^p \biggr)^{1/p}.
\end{align*}
\subsubsection{Feynman-Kac and Itô's Lemma}
Regarding proofs for this subsection, see \cite{karatzas1991brownian}.
\begin{theorem}[Feynman-Kac]\label{feynmankac}
Let $f: \mathbb{R}^d \rightarrow  \mathbb{R}$, $k:\mathbb{R}^d \rightarrow [0,\infty]$, and $g:  [0,T] \times \mathbb{R}^d \rightarrow \mathbb{R}$ be continuous functions satisfying, for a constant $\lambda \geq 1$,
$$
f(x) \in G_\lambda(\mathbb{R}^d) \quad \textnormal{or} \quad f(x) \geq 0; \quad \forall x \in \mathbb{R}^d
$$
and
$$
g(x) \in G_\lambda(\mathbb{R}^d) ; \quad 0 \leq t \leq T  \quad \textnormal{or} \quad g(t,x) \geq 0; \quad \forall x \in \mathbb{R}^d \,, 0 \leq t \leq T \text{.}
$$

 Suppose a continuous function $y: [0,T] \times \mathbb{R}^d \rightarrow \mathbb{R}$ is the expectation of a stochastic process
\begin{align}\label{feynmankacstochastic}
y(x,t) &= \mathbb{E}\biggl[ f(W_{T-t}) \, \exp\biggl\{ 
 -\int_0^{T-t} k(W_s) \, ds \biggr\} \notag \\  
 &\hspace{5mm} + g(t+\theta,W_\theta) \exp\biggl\{ -
 \int_0^{\theta}  k(W_s) \, ds \biggr\} \biggr] 
\end{align}
Then $u$ satisfies the partial differential equation
\begin{align}\label{feynmankacpde}
    - \frac{\partial y}{\partial t} + ky  = \frac{1}{2} \Delta y  + g &&\text{on }[0,T]  \times  \mathbb{R}^d  \\
    y(T,x) = f(x)  &&\text{for }x \in \mathbb{R^d}\text{.}
\end{align}
\end{theorem}
\begin{definition}[Quadratic Forms]\label{quadform}
For a square-integrable martingale $X$, we define the \textit{quadratic form} $\langle X \rangle$ to be the increasing process $A$ such that for a right-continuous martingale $M_t$,
$$
X^2_t = M_t + A_t \text{.}
$$
Given two square-integrable martingales $X$ and $Y$, we define the bilinear form $\langle \cdot , \cdot \rangle$ as
$$
\langle X,Y \rangle := \frac{1}{4} [ \langle X+Y \rangle_t - \langle X - Y \rangle_t ].
$$
\end{definition}
\begin{lemma}[Itô's Lemma]\label{itolemma}
Let $f:\mathbb{R}^d \rightarrow \mathbb{R}^d$ be a twice differentiable function and $X$ be of the form
$$
X_t = X_0 + M_t + B_t 
$$
where $M_t$ is a vector of martingales and $B_t$ is a vector of bounded variation processes. Then
\begin{align*}
f(t,X_t) &= f(0,X_0) +\int_0^t \frac{\partial}{\partial t} f(s,X_s)  ds + \sum_{i=1}^d \int_0^t \frac{\partial}{\partial x_i} f(s,X_s) dB_s^{(i)} \\
&\hspace{5mm} + \sum_{i=1}^d \int_0^t \frac{\partial}{\partial x_i} f(s,X_s) dM_s^{(i)} \\
&\hspace{5mm} + \frac{1}{2} \sum_{i=1}^d \sum_{j=1}^d \int_0^t \frac{\partial^2}{\partial x_i\partial x_j} f(s,X_s) d \langle M^{(i)},M^{(j)} \rangle_s .
\end{align*}
\end{lemma}
\begin{cor}\label{itoibp}[Integration by parts]
Let $X_t$, $Y_t$ be martingales of the form
$$
X_t = X_0 + M_t + B_t
$$
and 
$$
Y_t = Y_0 + N_t + C_t,
$$
where $M_t,N_t$ are local martingales and $B_t,C_t$ are bounded variation processes.
$$
X_tY_t  = X_0Y_0 + \int_0^t X_s dY_s + \int_0^t Y_s dX_s + \langle M,N \rangle_s 
$$
\end{cor}
\begin{definition}[Weak solutions]
In this paper, when we refer to a solution of an SDE, we consider weak solutions. Suppose $X_t$ satisfies a stochastic differential equation
$$
dX_t  = b(t,X_t) dt + \sigma(t,X_t) dW_t
$$
A \textit{weak solution} consists of a probability space $(\Omega,\mathcal{F},\mathbb{P})$, as well as a filtration $\{\mathcal{F}$, and processes $(X,W)$ where
\begin{enumerate}
    \item  $X = \{ X_t , \mathcal{F}_t; 0 \leq t<\infty \}$ is a continuous adapted $\mathbb{R}^d$-valued process,
    \item $W = \{ W_t,\mathcal{F}_t; 0 \leq t < \infty\}$ is an $r$-dimensional Brownian motion,
    \item $\mathbb{P}[ \int_0^t \{ b_i (s,X_s) + \sigma_{ij}^2 (s,X_s)\} ds < \infty ] = 1]$ for every $1 \leq i < d$, $1 \leq j \leq r$, and $0 \leq t < \infty$, and
    \item $X_t = X_0 + \inf_0^t b_i(s,X_s) d + \int_0^t \sigma(s,X_s) dW_s ; \quad 0 \leq t < \infty$.
\end{enumerate}
\end{definition}
\begin{lemma}\label{feynmankacv2}
Let $g \in G^\infty(\mathbb{R}^d)$ and suppose $u$ satisfies the partial differential equation
$$
\begin{cases}
             \frac{\partial}{\partial t} u(t,x) + \mathcal{L} u(t,x) = 0  & (t,x) \in  [0,T) \times \mathbb{R}^d, \\
             u(T,x) = g(x) & x \in \mathbb{R}^d .
       \end{cases}
$$
where the differential operator
$$
\mathcal{L} := \sum_{i=1}^d \mu_i \partial_i +  \sum_{1 \leq i,j \leq d} ( \sigma \sigma^T )_{ij} \partial^2_{ij}
$$
is the generator of the SDE $(\ref{generalsde})$.
Then 
$$
u = \mathbb{E} g (X_T (x)) ,
$$
where $X_t(x)$ is a process defined as the solution of an SDE of the form
\begin{equation}\label{generalsde}
dX_t  = \mu(X_t) dt + \sigma(X_t) dW_t
\end{equation}
with initial condition $X_0 = x$, where $\mu$ and $\sigma$ are Lipschitz.
\end{lemma}
\begin{proof}
    By applying Lemma \ref{itolemma} from time $t$ to $T$, we have
\begin{align*}
u(t,X_t) &= u(T,X_T) - \int_t^T \nabla u(s,X_s) \sigma(X_s) dW_s - \int_t^T \frac{\partial}{\partial t}u(s,X_s) ds - \int_t^T \mathcal{L} u(s,X_s) ds \\
&= u(0,X_0) - \int_t^T \nabla u(s,X_s) \sigma(X_s) dW_s .
\end{align*}
Then by taking expectation, we conclude
\begin{align*}
    u(t,X_t) &= \mathbb{E}[u(T,X_T)] = \mathbb{E}g(X_T(t,x)).
\end{align*}
\end{proof}
\subsection{Assumptions}
Here, we will in addition work under several assumed conditions:
\begin{enumerate}
    \item [(A1)] For $h \in (0,1)$, we consider learning rates $\eta^h_n$ that have the form
    $$
    \eta_n^h = \begin{pmatrix}
hu^1_{nh} & & & \\
& hu^2_{nh} & & \\
& & \ddots &  \\
 & &  & hu^d_{nh}
 \end{pmatrix} ,
    $$
    where $u^i_m$ is a family of functions in $C^\infty$ such that each $u$ is constant or strictly decreasing, and $\abs{u_t} \in (0,1] \quad \forall t \in [0,\infty)$. We can now write the generalized stochastic gradient descent equation as
\begin{equation}\label{modifiedlrsgd}
    \chi^h_{n+1} = \chi^h_n + \eta^h_n H_{\gamma (n)} (\chi_n^h) ,\quad \chi_0 = x \text{,}
\end{equation} and momentum-based SGD as 
\begin{equation}\label{modifiedlrmomentumsgd}
    \chi^h_{n+1} = \chi^h_n + \eta^h_n H_{\gamma (n)} (\chi_n^h) + \zeta_n ( \chi_n - \chi_{n-1} ), \quad (\chi_1, \chi_0) = x_1,x_0 \text{,}
\end{equation}
where for functions $v \in C^\infty$, 
$$
\zeta_n^h = \begin{pmatrix}
v^1_{nh} & & & \\
& v^2_{nh} & & \\
& & \ddots &  \\
 & &  & v^d_{nh}
 \end{pmatrix}
$$
    \item [(A2)] For any $r \in \Gamma$, $H_r (x) \in G_1(\mathbb{R}^d)$
    \item [(A3)] Define 
    $$
    \bar{H} = \mathbb{E}_{\gamma}H = \sum_{i \in \Gamma} \mathbb{P}[\gamma = i] H_i ,
    $$ 
    and $\Sigma = \mathbb{E}_\gamma(H_{\gamma(0)} - \bar{H})(H_{\gamma(0)} - \bar{H})^T$.\\
    $\bar{H}$ and $\sqrt{\Sigma}$ are Lipschitz continuous and all partial derivatives up to order $2$ are bounded.
\end{enumerate}
\section{Main Theorem}\label{sectionmaintheorem}
Let $\chi_n$ be the process described by (\ref{modifiedlrsgd}) for $n \geq 2$. 
We assume that learning rates are taken from the set
$$
\mathcal{H} := \{ h \in (0,1): T/h \in \mathcal{N} \}
$$
for a given time horizon $T >0$. The statement of the main result is as follows: 
\begin{theorem}\label{sgdmomentumapproxmain}
Assume $(A2),(A3)$, and that $\mathcal{\chi}$ has constant learning  rate and momentum parameter $\mathring{\eta}$ and $\mathring{\zeta}$, respectively. Define $j : \mathbb{N} \times \mathbb{R}^{2d} \rightarrow \mathbb{R}$,
$$
j^{(n)}(x) = f(x_1,...,x_d) + \sum_{i=1}^d \biggl[ (1+\mathring{\zeta}_n^i)\frac{x_i^2}{2 \mathring{\eta}_n^i} + \mathring{\zeta}_n^i \frac{x_{d+i}^2}{2\mathring{\eta}_n^i} - \frac{\mathring{\zeta}_n^i x_i}{\mathring{\eta}_n^i}  x_{d+i}\biggr]
$$
Then let $J^{(n)}_{\gamma(n)}(\chi_n) := \nabla j^{(n)}_{\gamma(n)} (\chi_n)$ and $\bar{J}_t$ is a continuous process such that for $t \in \mathbb{N}$, we have $\bar{J}_t = \mathbb{E}[J^{(t)}_{\gamma(t)}]$. Let $X$ be the solution of the ODE
$$
dX_t = U_t \bar{J}(X_t) \, dt
$$
with initial condition $X_0 = (x_1,x_0)$ for $x_1,x_0 \in \mathbb{R}^d$, where
$$U_t = \begin{pmatrix}
h\mathring{u}^1_n & & & & & \\
& \ddots & & & &\\
& & h\mathring{u}^d_n & & & \\
 & & & -h\mathring{u}^1_n/\mathring{v}_n ^1& &  \\
 & & & & \ddots & \\
 & & & & & -h\mathring{u}^d_n/\mathring{v}^d_n
\end{pmatrix} .$$
Then for all $g \in C^\infty (\mathbb{R}^d)$,
$$
\mathbb{E} g  (\chi_{T/h}^h) - g(X_T) = h \int_0^T  \phi(t) dt + \mathcal{O}(h^2),
$$
where $y_t(x) = g(X_T^t (x))$, and
$$
\phi_t (x) = \frac{1}{2} \textnormal{tr}[ \nabla^2 y_t (x) ((U_t \bar{J})(U_t \bar{J})^T)(x) ] + \partial_t \nabla y_t (x)^T(U_t \bar{J})(x) + \frac{1}{2} \partial_t^2 y_t(x) \text{.}
$$
Now suppose instead that $X$ is the solution of the SDE
$$
dX_t^h = u_t \bar{H}(X_t^h) dt + u_t \sqrt{h \Sigma (X_t^h) } dW_t
$$
with initial condition $X_0^h = (x_1,x_0)$ for $x_1,x_0 \in \mathbb{R}^d$, where
$$U_t = \begin{pmatrix}
h\mathring{u}^1_n & & & & & \\
& \ddots & & & &\\
& & h\mathring{u}^d_n & & & \\
 & & & -h\mathring{u}^1_n/\mathring{v}_n ^1& &  \\
 & & & & \ddots & \\
 & & & & & -h\mathring{u}^d_n/\mathring{v}^d_n
\end{pmatrix}  .$$
Then for all $g \in C^\infty (\mathbb{R}^d)$,
$$
\mathbb{E} g  (\chi_{T/h}^h) - g(X_T) = h \int_0^T  \varphi_t^g dt + \mathcal{O}(h^2),
$$
where $y_t(x) := \mathbb{E}[g(X_T^{h,t} (x))]$, and
$$
\varphi_t^h (x) = \frac{1}{2} \textnormal{tr}[ \nabla^2 y_t^h (x) ((U_t \bar{J})(U_t \bar{J})^T)(x) ] + \partial_t \nabla y_t^h (x)^T(U_t \bar{J})(x) + \frac{1}{2} \partial_t^2 y_t^h(x) \text{.}
$$ 
\end{theorem}
\begin{remark}
While the result above depends on a stochastic process with two inputs $(x_1,x_0)$, in practice, given $x_0$, we will obtain $x_1$ from a single step of standard SGD. Furthermore, in tackling Theorem \ref{sgdmomentumapproxmain}, we treat iterates as $2d$-dimensional to retain the values of two steps at a time. As an example, with one dimensional inputs, the process
$$
\chi_{n+1} = \chi_n - \eta_n f' (\chi_n) + \zeta_n ( \chi_n - \chi_{n-1} )
$$
becomes
$$
\chi_{n+1} = \begin{pmatrix}
    x_{n+1} \\
    x_n
\end{pmatrix} = \begin{pmatrix}
    x_n \\
    x_{n-1}
\end{pmatrix} + \begin{pmatrix}
     - \eta \nabla f'(x_n) + \zeta (x_n - x_{n-1}) \\
    x_n - x_{n-1}
\end{pmatrix} .
$$
If we wished to express this in the format given in (A1), we have
$$
x_{n+1}^h = \chi^h + \begin{pmatrix}
     \eta \\
    -\eta/\zeta
\end{pmatrix} J(\chi_n^h),
$$
where
$$
J(\chi_n^h) = \begin{pmatrix}
     - H(x_n) + \frac{\zeta}{\eta} (x_n - x_{n-1}) \\
    \frac{\zeta}{\eta}(x_n - x_{n-1})
\end{pmatrix},
$$
is the gradient of the objective function
$$
j(x,y) = f(x) + \zeta \frac{x^2}{2\eta} +\zeta \frac{y^2}{2\eta} - \frac{\zeta}{\eta} xy.
$$
\end{remark}
\section{ODE Approximation}\label{sectionodeproof}
The proof of the main theorem will be separated into two sections. Here we will prove the ODE case of Theorem \ref{sgdmomentumapproxmain}, which requires that we establish several growth properties. Firstly, we examine the standard case of SGD without a momentum parameter. This section follows the general approach of \cite{ankirchner2021stochasticapproximation}, with the key difference of multi-dimensionality of the learning rates $\eta$. Furthermore, we provide more precise estimates. The proof of the SGD case will be structured as illustrated in Figure \ref{figode}.

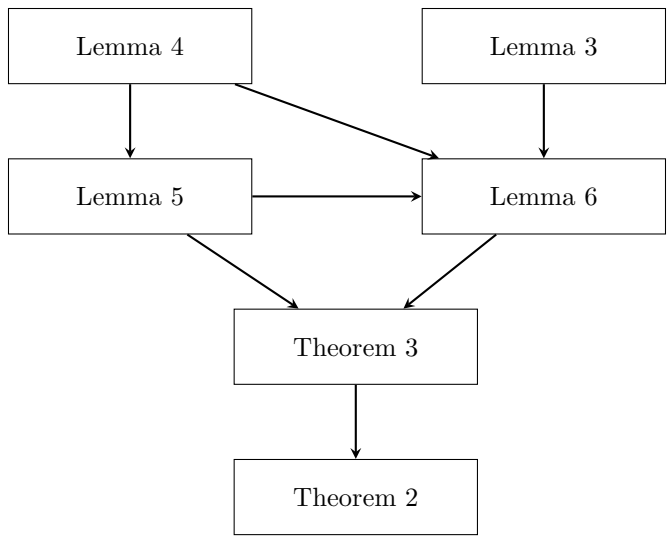
\begin{figure}[H]
\centering
\begin{tikzpicture}[node distance=2cm]

\node (lem3) [lemma, xshift=-3.5cm] {Lemma \ref{supbound}};
\node (lemy) [lemma, right of=lem3, xshift=3.5cm] {Lemma \ref{yfunctionspace}};
\node (lem4) [lemma, below of=lem3] {Lemma \ref{taylorapplication}};
\node (lem5) [lemma, below of=lemy] {Lemma \ref{differencebound}};
\node (thm3) [lemma, below of=lem4, xshift=3cm] {Theorem \ref{sgdodeapproxmain}};
\node (thm2) [lemma, below of=thm3] {Theorem \ref{sgdmomentumapproxmain}};

\draw [arrow] (lem3) -- (lem4);
\draw [arrow] (lem3) -- (lem5);
\draw [arrow] (lem4) -- (lem5);
\draw [arrow] (lemy) -- (lem5);
\draw [arrow] (lem4) -- (thm3);
\draw [arrow] (lem5) -- (thm3);
\draw [arrow] (thm3) -- (thm2);
\end{tikzpicture}

\caption{Flow of the proof to the Ordinary Differential Equation case}
\label{figode}
\end{figure}
The result for SGD is stated in Theorem \ref{sgdodeapproxmain}, which we will extend to derive the main result, Theorem \ref{sgdmomentumapproxmain}. Lemma \ref{supbound} is an important bound on the growth of the discrete process, as well as a key component in the proofs of Lemmas \ref{taylorapplication} and \ref{differencebound}.
In this section as well, we assume that learning rates are taken from the set
$$
\mathcal{H} := \{ h \in (0,1): T/h \in \mathcal{N} \}
$$
for a given time horizon $T >0$. For an ODE approximation of SGD, the main theorem takes the form below.
\begin{theorem}\label{sgdodeapproxmain}
Let the stochastic process $\chi_n$ 
be as defined in (\ref{modifiedlrsgd}). Assume $(A1),(A2),$ and $(A3)$, and let $X$ be the solution of
$$
dX_t =U_t \bar{H}(X_t) \, dt, \, X_0 = x ,
$$
where
$$U_t = \begin{pmatrix}
u^1_t & & & \\
& u^2_t & & \\
& & \ddots &  \\
 & &  & u^d_t 
\end{pmatrix}  .$$
Then for all $g \in C^\infty (\mathbb{R}^d)$,
$$
\mathbb{E} g  (\chi_{T/h}^h) - g(X_T) = h \int_0^T  \phi^g_t (X_t) + \frac{1}{2} \textnormal{tr} [ \nabla^2 y_t^g (X_t)U_t U_t \Sigma (X_t) ] \, dt + \mathcal{O}(h^2),
$$
where $y_t(x) = g(X_T^t (x))$, and
$$
\phi_t (x) = \frac{1}{2} \textnormal{tr}[ \nabla^2 y_t (x) ((U_t \bar{H})(U_t \bar{H})^T)(x) ] + \partial_t \nabla y_t (x)^T(U_t \bar{H})(x) + \frac{1}{2} \partial_t^2 y_t(x) \text{.}
$$
Note that $y$ is then the solution of the equation
\begin{equation}\label{feynmankaceqny}
\partial_t y_t (x) + \nabla y_t(x)^TU_t \bar{H} (x) = 0 , \quad y_t(x) = g(x) .
\end{equation}
\end{theorem}
\begin{theorem}\label{sdeexpindexswap}
Let $\ell \in \mathbb{N}$, $p \geq 1$, and $b,\sigma \in G_1 (\mathbb{R}^d) \cap \textnormal{Lip}^\ell$, $b$ is $\mathbb{R}^d$-valued, and $\sigma$ is $\mathbb{R}^d \times \mathbb{R}^d$-valued. Let $x \in  \mathbb{R}^d$, $s \in [0,T]$, and $X$ be the unique solution to the SDE
$$
dX_t = b_t(X_t) dt + \sigma_t(X_t) dW_t
$$
with initial condition $X_s = x$. Then $X$ is $\ell$-times continuously differentiable with respect to $x$ at any time $t \in [s,T]$, and for any multi-index $\alpha$ with $0 < \abs{\alpha} \leq \ell$, we have
$$
\partial_\alpha X_t = \psi_\alpha + \int_s^t \nabla b_u (X_u) \partial_\alpha X_u du + \int_s^t \nabla \sigma_u (X_u) \partial_\alpha X_u dW_u ,
$$
for some $\kappa \in \mathbb{N}$, $\| \psi_\alpha^*\|_p \in G_\kappa (\mathbb{R}^d)$ for all $p \geq 2$. Furthermore, for all $t \geq 0$,
$$
\partial_\alpha \mathbb{E}[X_t] = \mathbb{E}[ \partial_\alpha X_t].
$$
\end{theorem}
\begin{proof}
\cite{Kunita2004}
\end{proof}
\begin{lemma}\label{yfunctionspace}
Let $B,S \in \textnormal{Lip}^\ell \cap G_1(\mathbb{R}^d)$ and assume $(A1)$. Let $X$ be the solution of an SDE
$$
d X_t = u_t B(X_t)  dt + u_t S(X_t) W_t .
$$
Given $g: I \times \mathbb{R}^d \rightarrow \mathbb{R} \in G^\ell (\mathbb{R}^d) $, define $y _t^{i,h} := \mathbb{E}g^i(X_T^{h,t})$. Then $y \in G^\ell([0,T] \times \mathbb{R})$.
\end{lemma}
\begin{proof}
    Let $\alpha$ be a multi-index. Then by Theorem \ref{sdeexpindexswap}, it can be shown by induction that $\mathbb{E}\partial_\alpha y(X) = \partial_\alpha \mathbb{E} y(X)$. Then by the chain rule,
$$
\abs{\partial_\alpha v_t^{i,s}} = \mathbb{E}\abs{\partial_\alpha g^i (X_t^{i,s}} \leq \sum_{j=1}^{\abs{\alpha}} \| \nabla^j g^i (X)^* \|_2 \sum_{\mathcal{B} \in \mathcal{S}_j^\alpha} N(\alpha,\mathcal{B}) \prod_{\beta \in \mathcal{B}}\| \partial_\beta X^*\|_{2 \# \mathcal{B}},
$$
where $\mathcal{S}_j^\alpha$ is the set of partitions of $\alpha$ into $i$ multi-set multi-indices, $N(\alpha,\mathcal{B}) \in \mathbb{N}$, $\# \mathcal{B}$ is the size of the partition $\mathcal{B}$, and $\prod_{\beta \in \mathcal{B}}$ respects multiplicities of $\beta \in \mathcal{B}$. Then from $g \in G^\ell (\mathbb{R}^d)$ and \ref{sdeexpindexswap} we conclude that $y \in G^\ell (\mathbb{R}^d)$.
\end{proof}
\begin{lemma}\label{supbound}
Assume $(A1),(A2),$ and $(A3)$. The following bounds hold true:
    \begin{enumerate}
        \item For every $T > 0$ and $p \geq 1$ there exists a constant $C > 0$ such that
        $$
        \sup_{h \in (0,1)} \| \chi^h (x)^* \|_{p,\lfloor\frac{T}{h}\rfloor} \leq C(1 + \abs{x})
        $$
        for $x \in \mathbb{R}^d$.
        \item  For clearnesss of notation, we define $\tilde{y}_n := y_{nh}$
$$
\Delta \chi_n^{h,k} (x) := \tilde{y}^{h,k}_{n+1} (x) - \tilde{y}_n^{h,k} (x) .
$$
Then there exists a constant $C > 0$ such that
        $$
        \| \Delta \chi_n^{h,n} (x) \|_p \leq hC(1 + \abs{x})
        $$
        for all $h \in (0,1)$, $n \in \mathbb{N}$, and $x \in \mathbb{R}^d$. 
    \end{enumerate}
\end{lemma}
\begin{proof}
\textit{1.} Let $p \in \mathbb{N}$. For every $h \in (0,1)$, $n \in \mathbb{N}_0$, and $j \in \{ 1,2,...,d \}$, we expand
\begin{align*}
\abs{(\chi_{n+1}^h)_j}^p &= \abs{(\chi_n^h)_j + (hu_{nh}^j)H_{\gamma(n)}((\chi_{n+1}^h)_j)}^p \\
&\overset{\Delta}{\leq} \abs{(\chi_n^h)_j} + \sum_{i=1}^p {p \choose i} \abs{(\chi_n^h)_j}^{p-i} ( h u_{nh}^j )^i \abs{H_{\gamma(n)}((\chi_{n+1}^h)_j)}^p .
\end{align*}
Then with $\chi_{-1} = 0$,
\begin{align*}
\abs{\chi_{n+1}^h}^p &= \abs{(\chi_n^h) + (hU_{nh})H_{\gamma(n)}(\chi_{n+1}^h)}^p \\
&\leq \abs{(\chi_n^h)_j}^p + \sum_{i=1}^p {p \choose i} \abs{(\chi_n^h)_j}^{p-i}\abs{( (U_{nh})^i H_{\gamma(n)})(\chi_{n+1}^h)}^p  \\
&\leq \abs{(\chi_n^h)_j}^p + \sum_{i=1}^p {p \choose i} \abs{(\chi_n^h)_j}^{p-i}  \max_{j \in \{1,2,...,d\}} (h u_{nh}^j )^i \abs{H_{\gamma(n)}(\chi_{n+1}^h)}^p .
\end{align*}
Now, for $i \in \{ 1,... p \}$, $h \in (0,1)$, and $n \in \mathbb{N}_0$,
\begin{align*}
\| ( \abs{\chi^h}^{p-i} \abs{H_{\gamma (0)}(\chi^h)^i} )^* \|_{1,n} &= \| ( \abs{\chi^h}^{p-i} \| H \|^i_{G_1} (1 + \abs{\chi^h})^i )^* \|_{1,n} &&\text{from (A2)} \\
&\leq \| ( \abs{\chi^h}^{p-i} \| H \|^i_{G_1} 2^{i-1} (1 + \abs{\chi^h}^i) )^* \|_{1,n} \\
&\leq \frac{1}{2} c^i \| ( \abs{\chi^h}^{p-i} + \abs{\chi^h}^p )^* \|_{1,n} \\
&\leq \frac{1}{2} c^i \| ( 2 + \abs{\chi^h}^p + \abs{\chi^h}^p )^* \|_{1,n} \\
&\leq c^i ( 1 + \|(\chi^h)^*\|^p_{p,n} ) ,
\end{align*}
where $c := 2 \| H \|^i_{G_1}$. Thus, 
\begin{align*}
    \| (\chi^h)^* \|^p_{p,n+1} &\leq \mathbb{E}  \max_{n' \in \{ -1,...,n \}} \abs{ \chi^h_{n'} }^p \\
    &\hspace{5mm} + \mathbb{E} \max_{n' \in \{ -1,...,n \}} \sum_{i=1}^p {p \choose i} \abs{(\chi_{n'}^h)_j}^{p-i} \max_{j \in \{1,2,...,d\}} (h u_{n'h}^j )^i \abs{H_{\gamma(n')}(\chi_{n'+1}^h)}^p \\
    &\leq \| (\chi^h)^* \|^p_{p,n'} + Ch (1 + \|(\chi^h)^* \|^p_{p,n'} ) \\
    &= (1 + Ch) \| (\chi^h)^* \|^p_{p,n} + Ch, 
\end{align*}
where $C := \sum_{i=1}^p {p \choose i} c^i$. Therefore we can obtain, by induction over $n$, the following bound. For all $h  \in (0,1)$ and $n \in \mathbb{N}$, 
$$
\| (\chi^h)^* \|^p_{p,n} \leq (1 + Ch)^n \| (\chi^h)^* \|^p_{p,0} + Ch \sum_{i=0}^{n-1} ( 1 + Ch )^i  .
$$
Applying this to the first bound in the lemma, we have, for $x \in \mathbb{R} $ and $h \in (0,T)$
\begin{align*}
    \| \chi^h (x)^* \|_{p,\lfloor \frac{T}{h} \rfloor} &\leq (1 + Ch)^{ \lfloor \frac{T}{h} \rfloor } \abs{x} ^p + Ch \sum_{i=0}^{\lfloor \frac{T}{h} \rfloor-1} ( 1 + Ch )^i   \\
    &\leq (1 + Ch)^{  \frac{T}{h}  } \abs{x} ^p + Ch \frac{T}{h} ( 1 + Ch )^\frac{T}{h} \\
    &= (CT + \abs{x}^p)(1 + Ch)^{\frac{T}{h}} .
\end{align*}
For any $p \geq 1$, we have $\| \chi \|_p \leq \| \chi \|_{ \lceil p \rceil }$, thus the result follows.\\~\\
\textit{2.} For all $x \in \mathbb{R}^d$ and $h \in (0,1)$, we have
$$
\| \Delta \chi_n^{h,n}(x) \|_p = \| \eta_n^h H(x) \|_p \leq h \| H \|_{G_1}(1 + \abs{x}).
$$
\end{proof}
\begin{lemma}\label{taylorapplication}
There exists a bounded function $\xi : \mathcal{H} \rightarrow \mathbb{R}$ such that $\forall h \in \mathcal{H}$,
$$
\mathbb{E} g (\chi^h_{T/h}) - g(X_T) - h^2 \sum_{k=0}^{\frac{T}{h} - 1} \mathbb{E} \Phi_{kh}(\chi^h_k ) = h^2 \xi(h),
$$ 
where
$$
\Phi_t(x) := \phi_t (x)+ \frac{1}{2}\textnormal{tr}[\nabla^2 y_t (x) \Sigma (x)] .
$$
\end{lemma}
\begin{proof}
    Since $y \in C^\infty$, we have via Taylor expansion
\begin{align*}
    y_{t+h}(x+\delta) - y_t(x) &= h \partial_ty_t (x) + \nabla y_t (x)^T \delta  \\
    &\hspace{5mm} + \frac{1}{2}\textnormal{tr}[ \nabla^2 y_t(x) \delta \delta^T] + h \partial_t \nabla y_t(x)^T \delta + \frac{h^2}{2} \partial^2 _t y_t(x) \\
    &\hspace{5mm} + R^h(\delta), 
\end{align*}
where 
$$
R^h(\delta):= \sum_{k=0}^3 \sum_{\# \beta = 3 - k} \frac{1}{\beta ! k !} \partial_t^k \partial_\beta y_{t + \theta h} (x + \theta \delta) h^k \delta^\beta
$$
for some $\theta \in (0,1)$, all $h \in (0,1)$, and $\delta \in \mathbb{R}^d$.
Then taking $t = kh$, $\delta = \Delta \chi_k^h$ and applying expectation yields
$$
\mathbb{E}y_{(k+1)h}(\chi^h_{k+1}) - \mathbb{E} y_{kh}(\chi_k^h) = h Z_1^h + h^2( Z_{xx}^h + Z_{tx}^h + Z_{tt}^h ) + \mathbb{E} R^h(\Delta \chi_k^h) ,
$$
where
\begin{align*}
    Z^h_1 &:= \mathbb{E}[ \partial_t y_{kh}(\chi_k^h) + h^{-1} \nabla y_{kh} ( \chi_k^h )^T \Delta \chi_k^h ] , \\
    Z^h_{xx} &:= \frac{1}{2} u^2_{kh} \mathbb{E} \textnormal{tr} [ \nabla^2 y_{kh} ( \chi^h_k ) \cdot \\
    &\hspace{5mm} ((\bar{H}( \chi^h_k ) + ( H_{\gamma(0)} - \bar{H} )( \chi^h_k ))((\bar{H}( \chi^h_k ) + ( H_{\gamma(0)} - \bar{H} )( \chi^h_k ))^T) ] ,\\
    Z^h_{tx} &:= u_{kh} \mathbb{E} \textnormal{tr} [ \nabla y_{kh} (\chi^h_k)^T \bar{H}(\chi_k^h) ], \\
    Z_{tt}^h &:= \frac{1}{2} \mathbb{E}[ \partial_t^2 y_{kh} ( \chi_k^h ) ] .
\end{align*}
Note that
$$
\mathbb{E} \Delta \chi_n^{h,n} = \eta_n^h \bar{H},
$$
so we may simplify $Z_1$ to
\begin{align*}
Z^h_1 &= \mathbb{E}[\partial_t y_{kh}(\chi_k^h) + h^{-1} \nabla y_{kh} ( \chi_k^h )^T \eta_n^h \bar{H}(\chi_k^h)] \\
&= \mathbb{E}[\mathbb{E}[\partial_t y_{kh}(\chi_k^h) + \nabla y_{kh} ( \chi_k^h )^T u_n^h \bar{H}(\chi_k^h)]| \chi_k^h ] \\
&= 0 &&\text{by (\ref{feynmankaceqny})}
\end{align*}
Furthermore,
\begin{align*}
Z_{xx}^h &= \frac{1}{2} u^2_{kh} \mathbb{E} \textnormal{tr} [ \nabla^2 y_{kh} ( \chi^h_k ) ((\bar{H}( \chi^h_k )\bar{H}( \chi^h_k ) + 2\bar{H}( \chi^h_k )( H_{\gamma(0)} - \bar{H} ) \\
&\hspace{5mm}+  ( H_{\gamma(0)} - \bar{H} )( \chi^h_k )( H_{\gamma(0)} - \bar{H} )( \chi^h_k ))) ] \\
&= \frac{1}{2} u^2_{kh} \mathbb{E} \textnormal{tr} [ \nabla^2 y_{kh} ( \chi^h_k ) ((\bar{H}( \chi^h_k )\bar{H}( \chi^h_k ) +  ( H_{\gamma(0)} - \bar{H} )( \chi^h_k )( H_{\gamma(0)} - \bar{H} )( \chi^h_k ))) ] \\
&= \frac{1}{2} u_{kh}^2 \mathbb{E} \textnormal{tr} [ \nabla^2 y_{kh} (\chi_k^h) (\bar{H}\bar{H}^T + \Sigma)(\chi_k^h) ] \text{.} 
\end{align*}
The partial derivatives of $y$ are in $G_\kappa$, thus by $(A3)$ we have that $Z_{xx}^h$ has at most polynomial growth. The same is true of $Z_{tt}^h$ and furthermore by $(A2)$, of $Z_{tx}^h$. Therefore, $h^2 ( Z_{xx}^h + Z_{tx}^h + Z_{tt}^h ) = \mathcal{O}(h^2)$.\\~\\
In addition, for $k \in \{ 0,...,3\}$ and $\#\beta = 3 - k$,
\begin{align*}
    \mathbb{E} ( \abs{\bar{H} (\chi_n^h)^\beta }  )^{1 / \# \beta} &\leq \sup_{h \in (0,1)} \| \bar{H}( \chi^h )^* \|_{ \#  \beta , \lfloor \frac{T}{h} \rfloor } \\
    &\leq \| \bar{H} \|_{G_1} \biggl( 1 + \sup_{h \in (0,1)} \| ( \chi^h )^* \|_{\# \beta , \lfloor \frac{T}{h} \rfloor } \biggr) \\
    &\leq c( 1 + \abs{\chi_0} ). &&\text{by Lemma \ref{supbound}}
\end{align*}
Then
\begin{align*}
\mathbb{E} h^k ( \Delta\chi_n^h )^\beta  &= h^k h^{3-k} ( u_kh)^{3-k} \mathbb{E}\bar{H}( \chi_n^h )^\beta  \\
&= \mathcal{O}(h^3). &&\text{as } \abs{u_h} \leq1
\end{align*}
Since $\partial_t \partial_\alpha^{2 - k} y \in G ( [0,T] \times \mathbb{R}^d  )$ for all $k \in \{ 0,1,2 \}$, we have $\mathbb{E} R^h ( \Delta \chi^h_n ) = \mathcal{O}(h^3)$. Thus,
\begin{align*}
    \mathbb{E} g(\chi_{T/h}^h) - g(X_T) &= \mathbb{E}y_T ( \chi^h_{T/h} ) - \mathbb{E} y_0(\chi_0) \\
    &= \sum_{k=0}^{\frac{T}{h} - 1} \mathbb{E}y_{(k+1)h}(\chi_{(k+1)}^h) - \mathbb{E}y_{kh}( \chi_k^h ) \\
    &= h^2 \sum_{k=0}^{\frac{T}{h} - 1} \mathbb{E}\Phi_{kh}( \chi_k^h ) + \mathcal{O}(h^2)
\end{align*}
for all $h \in \mathcal{H}$.
\end{proof}
\begin{lemma}\label{differencebound}
There exists a constant $C > 0$ such that
$$
\sum_{n=0}^{\frac{T}{h} - 1} \abs{\mathbb{E}\Phi_{nh}(\chi_n^h) - \Phi_{nh}(X_nh)} \leq C
$$
for all  $h \in \mathcal{H}$.
\begin{proof}
    Recall the function $\xi$ from Lemma \ref{taylorapplication}. Then
\begin{align*}
    \sum_{n=0}^{\frac{T}{h} - 1} \abs{\mathbb{E}\Phi_{nh}(\chi_n^h) - \Phi_{nh}(X_nh)} &\leq h^2 \sum_{n=0}^{\frac{T}{h} - 1} \sum_{k=0}^{n-1} \mathbb{E}\biggl| \frac{1}{2}u_kh^2 \textnormal{khr}[\nabla^2 \Phi_{nh}(X^{kh}_{nh}) (\bar{H}\bar{H}^T + \Sigma )(\chi_k^h) ] \\  
    &\hspace{5mm} + u_kh \partial_kh \nabla \Phi_{nh}(X^{kh}_{nh}) (\chi_k^h) \bar{H}(\chi_k^h) + \frac{1}{2}\partial^2_kh \Phi_{nh}(X^{kh}_{nh}) (\chi_k^h) \biggr| + h\xi_n(h) \\
    &\leq C(1 + \max_{n,k} \mathbb{E}\biggl| \frac{1}{2}u_kh^2 \textnormal{khr}[\nabla^2 \Phi_{nh}(X^{kh}_{nh}) (\bar{H}\bar{H}^T + \Sigma )(\chi_k^h) ] \\  
    &\hspace{5mm} + u_kh \partial_kh \nabla \Phi_{nh}(X^{kh}_{nh}) (\chi_k^h) \bar{H}(\chi_k^h) + \frac{1}{2}\partial^2_kh \Phi_{nh}(X^{kh}_{nh}) (\chi_k^h) \biggr| ) 
\end{align*}
for some $C > 0$ and all $h \in (0,1)$. Then, with
\begin{align*}
M &:= \max_{n,k} \mathbb{E}\biggl| \frac{1}{2}u_kh^2 \textnormal{khr}[\nabla^2 \Phi_{nh}(X^{kh}_{nh}) (\bar{H}\bar{H}^T + \Sigma )(\chi_k^h) ] \\  
    &\hspace{5mm} + u_kh \partial_kh \nabla \Phi_{nh}(X^{kh}_{nh}) (\chi_k^h) \bar{H}(\chi_k^h) + \frac{1}{2}\partial^2_kh \Phi_{nh}(X^{kh}_{nh}) (\chi_k^h) \biggr| ),
\end{align*}
we have
\begin{align*}
    M &\leq \| \Phi \|_{G_\kappa} \biggl( 1 + \sup_{h \in (0,1)} \| (\chi^h)^* \|_1^\kappa \biggr) &&\text{from } \Phi \in G([0,T]\times \mathbb{R}^d) \\
    &\leq  \bigl[ C'(1 + \abs{\chi_0}) \bigr] ^\kappa &&\text{by Lemma \ref{supbound}}\\
    &\leq 2^{\kappa-1} C' ( 1 + \abs{\chi_0}^\kappa ) \\
    &\leq C ( 1 + \abs{\chi_0}^\kappa ) 
\end{align*}
for some $C > 0$, $\kappa \in \mathbb{N}$, and all $h \in (0,1)$.
\end{proof}

\end{lemma}
\subsection{Proof of Theorem \ref{sgdodeapproxmain}}
Fix $g \in C^\infty(\mathbb{R}^d)$ and let $y_t(x) = g(X_T^t(x))$, where $X$ is the solution of the equation
$$
dX_t =U_t \bar{H}(X-t) \, dt.
$$
By Lemma \ref{taylorapplication},
$$
\mathbb{E}g(\chi_{T/h}^h) - g(X_T) = h \sum_{n=0}^{\frac{T}{h} - 1} h \mathbb{E} \Phi_{nh} (\chi_n^h) + \mathcal{O} (h^2).
$$
We have that
\begin{align*}
\sum_{n=0}^{\frac{T}{h}- 1} h \mathbb{E} \Phi_{nh} (\chi_n^h) &= \int_0^T \Phi_t (X_t) \, dt + h \sum_{n=0}^{\frac{T}{h} - 1} \Bigl[ \mathbb{E} \Phi_{nh}(\chi_n^h) - \Phi_{nh}(X_nh) \Bigr] \\
&\hspace{5mm} + \sum_{n=0}^{\frac{T}{h} - 1} h \Phi_{nh}(X_{nh}) - \int_0^T  \Phi_t (X_t) \, dt.
\end{align*}
Then using Lemma \ref{differencebound}, we have
\begin{align*}
    h \sum_{n=0}^{\frac{T}{h} - 1} \abs{\mathbb{E}\Phi_{nh}(X_{nh}) - \Phi_{nh}(X_nh)} &\leq hC, \\
    \abs{\sum_{n=0}^{\frac{T}{h} - 1}h \Phi_{nh} - \int_0^T \Phi_t (X_t) \, dt} &\leq hC' ,
\end{align*}
where $C,C' > 0$ are constants. Therefore,
$$
\mathbb{E}g(\chi_{T/h}^h) - g(X_t) =  h \int_0^T \Phi_t (X_t) \,dt + \mathcal{O}(h^2).
$$  
In order to extend the above to the main theorem, we require that various bounds still hold under a time-dependent objective function $J$. Many of the lemmas follow as above if we are able to verify that $J$ satisfies similar properties to $H$, hence the following lemma. 
\begin{lemma}\label{varyingobjectivefunction}
 Asssume (A1),(A2), and (A3). Then
\begin{enumerate}
    \item [\textnormal{(A2$^*$)}] For any $r \in \Gamma$, $n \in [0,T]$, $J^{(n)}_r (x) \in G_1(\mathbb{R}^{2d})$, and there exists a constant $C$ such that $C = \sup_{n \geq 0} \| J^{(n)}_r (x) \|_{G_1}$.
    \item [\textnormal{(A3$^*$)}]  Define 
    $$
    \bar{J} = \mathbb{E}_{\gamma}J = \sum_{i \in \Gamma} \mathbb{P}[\gamma = i] J_i ,
    $$ 
    and $E = \mathbb{E}_\gamma(J_{\gamma(0)} - \bar{J})(J_{\gamma(0)} - \bar{J})^T$.\\
    $\bar{J}$ and $\sqrt{E}$ are Lipschitz continuous and all partial derivatives up to order $2$ are bounded.
\end{enumerate}
\end{lemma}
\begin{proof}
\textnormal{(A2$^*$)} Let $r \in \Gamma$ and $C$ be such that 
$$
\abs{H_r(x)} \leq C (1 + \abs{x})
$$
For $x = \begin{pmatrix}
    x_2 \\
    x_1
\end{pmatrix}$ with $x_1,x_2 \in \mathbb{R}^d$, we have $\abs{x_1 - x_2} \leq 2 \abs{x}$. Furthermore, from (A1), as $\abs{u^i},\abs{\zeta^i} \leq 1$, then
$$
\abs{J_r(x)} \leq C (1 + \abs{x_1}) + 2 \abs{x_1 - x_2} 
\leq (C + 4) (1 + \abs{x}),
$$
thus $J_r(x) \in G_1(\mathbb{R}^{2d})$.\\~\\
(A3$^*$) Let $L$ be the Lipschitz constant of $\bar{H}$. Then, by the same properties used in the proof of (A2$^*$)
\begin{align*}
    \abs{\partial_\alpha J(x) - \partial_\alpha J(y)} &\leq \abs{\partial_\alpha  \begin{pmatrix}
        \bar{H}(x_2) + \eta(x_2 - x_1) \\
        x_2 - x_1 
    \end{pmatrix} - \partial_\alpha \begin{pmatrix}
        \bar{H}(y_2) + \eta(y_2 - y_1) \\
        y_2 - y_1 
    \end{pmatrix} } \\
    &\leq \abs{\partial_\alpha  \begin{pmatrix}
        \bar{H}(x_2) - \bar{H}(y_2) + \eta(x_2 - x_1 - y_2 + y_1) \\
        x_2 - x_1 - y_2 + y_1 
    \end{pmatrix} } \\
    &\leq \abs{\partial_\alpha  \begin{pmatrix}
        \bar{H}(x_2) - \bar{H}(y_2) \\
        0 
        \end{pmatrix}} + \abs{ \begin{pmatrix} \eta(x_2 - x_1 - y_2 + y_1) \\
        x_2 - x_1 - y_2 + y_1 
    \end{pmatrix} } \\
    &\leq (2L + 4)\abs{x - y}  .
\end{align*}
Now, note that 
$$
J_{\gamma}(x) - \bar{J}(x) = \begin{pmatrix}
    H_{\gamma}(x_2) - (x_2 - x_1) \\
    x_2 - x_1 \end{pmatrix} - \begin{pmatrix}
    \bar{H}(x_2) - (x_2 - x_1) \\
    x_2 - x_1
    \end{pmatrix} = \begin{pmatrix}
    H_{\gamma}(x_2) - \bar{H}(x_2) \\
    0 \end{pmatrix} .
    $$
Thus,  $E   = \Sigma(x_2)$ and the proof concludes by $(A3)$.
\end{proof}
With this result, the previous lemmas used in the proof of Theorem \ref{sgdodeapproxmain} can be adapted to the time-dependent objective function $J$.  
\subsection{Proof of Theorem \ref{sgdmomentumapproxmain}, ODE Case}
\begin{proof}
For $n \geq 2$, let $\tilde{\chi}_n$ be the process described by (\ref{modifiedlrsgd}) with learning rate and momentum parameter $\mathring{\eta}$ and $\mathring{\zeta}$, respectively. Suppose we are given initial conditions $x_1, x_0$. We rewrite this as a $2d$-dimensional process $\chi_n$,
where
\begin{align*}
\chi^h_{{n+1}_{(1)}} &= \chi^h_{n_{(1)}} + \mathring{\eta}^h_n H_{\gamma (n)} (\chi_{n_{(1)}}^h) + \mathring{\zeta}_n ( \chi^h_{n_{(1)}} - \chi^h_{n_{(2)}} ) \\
\chi^h_{{n+1}_{(2)}} &= \chi^h_{n_{(1)}}
\end{align*}
with initial condition $\chi_0 = (x_1,x_0)$. Here, a subscript $(1)$ indicates the first $d$ dimensions and $(2)$ the final $d$ dimensions. In other words,
\begin{equation}
    \chi_{n+1}^h = \chi_n^h + \eta_n^h J^{(n)}_{\gamma(n)}(\chi^h_n),
\end{equation}
where
$$
\eta^h_n J^{(n)}_{\gamma(n)}(\chi_n^h) = \begin{pmatrix}
\mathring{\eta}_n^h H_{\gamma(n)} (\tilde{\chi}_n^h) + \mathring{\zeta}_n ( \tilde{\chi}^h_n - \tilde{\chi}^h_{n-1} )   \\
 \tilde{\chi}^h_n - \tilde{\chi}^h_{n-1}
\end{pmatrix} 
$$
Then $J^{(n)}_{\gamma(n)}(\chi_n) = \nabla j^{(n)}_{\gamma(n)} (\chi_n)$, where given $x \in \mathbb{R}^{2d}$,
$$
j^{(n)}(x) = f(x_1,...,x_d) + \sum_{i=1}^d \Bigl[ (1+\mathring{\zeta}_n^i)\frac{x_i^2}{2 \mathring{\eta}_n^i} + \mathring{\zeta}_n^i \frac{x_{d+i}^2}{2\mathring{\eta}_n^i} - \frac{\mathring{\zeta}_n^i x_i}{\mathring{\eta}_n^i}  x_{d+i}\Bigr]
$$
and
$$
\eta^h_n = \begin{pmatrix}
h\mathring{u}^1_n & & & & & \\
& \ddots & & & &\\
& & h\mathring{u}^d_n & & & \\
 & & & -h\mathring{u}^1_n/\mathring{v}_n ^1& &  \\
 & & & & \ddots & \\
 & & & & & -h\mathring{u}^d_n/\mathring{v}^d_n
\end{pmatrix}.
$$
Thus, we may apply Theorem \ref{sgdodeapproxmain} with $J$ in place of $H$.
\end{proof}
\section{SDE Approximation}\label{sectionsdeproof}
The method above can be extended to a Stochastic Differential Equation (SDE) approximation by considering the equation
\begin{equation}\label{sdeapproxequation}
    d X^h_t =  U_t \bar{H} (X_t^h) \, dt + U_t \sqrt{h\Sigma(X^h_t)}  \, dW_t.
\end{equation}
As with the ODE approximation, we will first prove the form of the above theorem for standard SGD, which takes the form below.
\begin{theorem}\label{momentumsderesult}
Assume $(A1)$,$(A2)$, and $(A3)$. For all $h \in \mathcal{H}$ denote by $X^h$ the solution of the SDE
$$
dX_t^h = u_t \bar{H}(X_t^h) dt + u_t \sqrt{h \Sigma (X_t^h) } dW_t
$$
with initial condition $X_0^h = x$. Then for all $g \in G^\infty (\mathbb{R}^d)$,
$$
\mathbb{E}g(\chi_{T/h}^h) =  h  \int_0^T \varphi_t^g dt + \mathcal{O} (h^2)
$$
for all $h$ such that $\frac{T}{h} \in \mathbb{N}$. 
\end{theorem}

The proof of the Theorem will again require a combination of results established within this section, following the structure as in Figure \ref{figsde}.

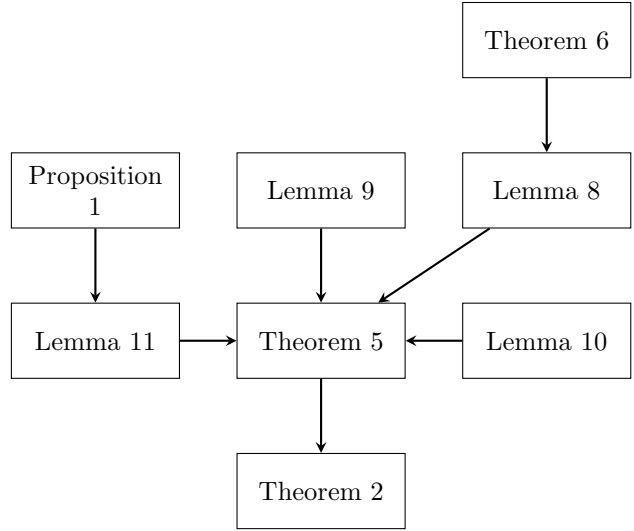
\begin{figure}[H]
\centering
\begin{tikzpicture}[node distance=2cm]

\node (lem2) [lemma2, xshift=3.5cm] { Theorem \ref{basesdebound}};
\node (lem4) [lemma2, below of=lem2, xshift=-3cm] {Lemma \ref{taylorapplicationsde} };
\node (thm1) [lemma2, below of=lem2] {Lemma \ref{sdedifferencebound}};
\node (lem5) [lemma2, below of=thm1] {Lemma \ref{itotaylorexpansion}};
\node (lem6) [lemma2, below of=lem2, xshift=-6cm]  {Proposition \ref{expdifforderswap}};
\node (lem7) [lemma2, below of=lem6] {Lemma \ref{ding2}};
\node (thm2) [lemma2, below of=lem4, xshift=0cm] {Theorem \ref{momentumsderesult}};
\node (thm3) [lemma2, below of=thm2] {Theorem \ref{sgdmomentumapproxmain}};

\draw [arrow] (lem5) -- (thm2);
\draw [arrow] (lem2) -- (thm1);
\draw [arrow] (thm1) -- (thm2);
\draw [arrow] (thm2) -- (thm3);
\draw [arrow] (lem7) -- (thm2);
\draw [arrow] (lem4) -- (thm2);
\draw [arrow] (lem6) -- (lem7);

\end{tikzpicture}
\caption{Flow of the proof to the Stochastic Differential Equation case}
\label{figsde}
\end{figure}
The SDE case of Theorem \ref{sgdmomentumapproxmain} is an extension of the SDE case for constant learning rates. Lemma \ref{taylorapplicationsde} is analogous in this case to Lemma \ref{taylorapplication}, with the key difference being the introduction of dependency on $h$. As a consequence, $y$ satisfies the time-dependent Feynman-Kac equation
\begin{equation}\label{feynmankacsde}
    \partial_t y_t(x)  + \nabla y_t (x)^T u_t \bar{J}(x) + \frac{1}{2} hu^2_t \textnormal{tr}(\nabla^2 y_t (x) \Sigma (x)) = 0 , \quad y_T (x) = g(x).
\end{equation}
Thus, we require some additional results to remove the $h$-reliance. 
\begin{theorem}\label{basesdebound}
Let
$$
b: (0,1) \times [0,T] \times \mathbb{R}^d \rightarrow \mathbb{R}^d 
$$
be a family of functions such that for every $h \in (0,1)$, $b^h \in G_1(X) \cap \text{Lip}^\ell$, and in addition let
$$
    \sigma: [0,T] \times \mathbb{R}^d \rightarrow \mathbb{R}^{d \times d} \in G_1(X) \cap \text{Lip}^\ell 
$$
and $L$ be the sum of the Lipschitz constants of $b$ and $\sigma$. Then the stochastic differential equation
$$
dX_t^h = b_t^h (X_t^h) dt + \sqrt{h} \sigma_t (X_t^h) dW_t
$$
admits a unique solution $X$, with the property furthermore there exists 
$$
\| X^* \|_{p,T} \leq C(1 + \abs{p})
$$    
for some constant $C$ depending only on $L$ and $T$.
\end{theorem}
\begin{proof}
    See \cite{Kunita2004}, Theorem 3.1.
\end{proof}
\begin{lemma}\label{sdedifferencebound}
    Let
$$
b: (0,1) \times [0,T] \times \mathbb{R}^d \rightarrow \mathbb{R}^d 
$$
be a family of functions such that for every $h \in (0,1)$, $b^h \in G_1(X) \cap \text{Lip}^\ell$ with Lipschitz constant $L_1$, and in addition let
$$
    \sigma: [0,T] \times \mathbb{R}^d \rightarrow \mathbb{R}^{d \times d} \in G_1(X) \cap \text{Lip}^\ell 
$$
with Lipschitz constant $L_2$. Now define $X$ as the unique solution to the stochastic differential equation
\begin{equation}\label{sdestandardform}
dX_t^h = b_t^h (X_t^h) dt + \sqrt{h} \sigma_t (X_t^h) dW_t
\end{equation}
with initial condition $X_0 = x$. Then there exists a function $C \in G_1(\mathbb{R}^d)$ such that for all $h \in (0,1)$ and $n \in \{ 0,..., \lfloor T/h \rfloor \}$,
$$
\| X_{nh}^h (x) - x \| \leq hC,
$$
where $\| C \|_{G_1}$ depends only on $T$, $p$, and $L := L_1 + L_2$.
\end{lemma}
\begin{proof}
By integrating the SDE (\ref{sdestandardform}), we have
\begin{equation*}
    \| X_{nh}^h(x) - x\|_p \leq \biggl\| \int_{nh}^{(n+1)h} b_s^h(X_s) d s \biggr\|_p + \sqrt{h} \biggl\| \int_{nh}^{(n+1)h} \sigma(X^h_s)  dW_s \biggr\|_p .
\end{equation*}
We may bound the first term by
\begin{align*}
    \biggl\| \int_{nh}^{(n+1)h} b_s^h(X_s) d s \biggr\|_p &= \biggl( \mathbb{E} \biggl( \int_{nh}^{(n+1)h} \abs{b_t^h (X_t)} dt \biggr)^p \biggr)^{1/p} \\
    &\leq h^{1 - \frac{1}{p}} \biggl( \int_{nh}^{(n+1)h} \mathbb{E}\abs{b_t^h (X_t)}^p dt \biggr)^{1/p} &&\text{by Jensen's inequality}\\
    &\leq h^{1 - \frac{1}{p}} \biggl( \int_{nh}^{(n+1)h} \mathbb{E} \sup_{t,h} \abs{b_t^h (X_t)}^p dt \biggr)^{1/p} \\
    &\leq h \| b(X)^* \|_p \\
    &\leq h c_1 (1 + \abs{x}) &&\text{by Theorem \ref{basesdebound}}
\end{align*}
where $c_1$ depends only on $T$ and $L$, and the second by
\begin{align*}
    \sqrt{h} \biggl\| \int_{nh}^{(n+1)h} \sigma(X^h_s)  dW_s \biggr\|_p 
&= \sqrt{h}  \biggl( \mathbb{E} \biggl(
\int_{nh}^{(n+1)h} \sigma_t (X^h_t) dW_t \biggr)^p \biggr)^{1/p} \\
    &\leq \sqrt{h} \biggl( \biggl(\frac{p(p-1)}{2}\biggr)^{p/2} h^{\frac{p-2}{2}}   \mathbb{E} \biggl[ 
\int_{nh}^{(n+1)h} \sigma_t (X^h_t)^p dt   \biggr] \biggr)^{1/p} &&\text{by Ito's isometry} \\
&\leq h^{1 - \frac{1}{p}} \sqrt{\frac{p(p-1)}{2}} \biggl( \mathbb{E} \biggl[ 
\int_{nh}^{(n+1)h} \sigma_t (X^h_t)^p dt   \biggr] \biggr)^{1/p} \\
&\leq h \sqrt{\frac{p(p-1)}{2}} \| \sigma (X^h)^* \|_p \\
&\leq h \sqrt{\frac{p(p-1)}{2}} c_2 (1 + \abs{x}) &&\text{by Theorem \ref{basesdebound}}
\end{align*}
for some constant $c_2$.
Thus,
\begin{align}
    \| X_{nh}^h(x) - x\|_p &\leq h C, 
\end{align}
where 
$$
\| C \|_{G_1} = c_1 + c_2\sqrt{\frac{p(p-1)}{2}}
$$
depends on $T,L,$ and $p$.
\end{proof}
\begin{lemma}\label{taylorapplicationsde}
Let $\xi : \mathcal{H} \rightarrow \mathbb{R}$ be the function such that $\forall h \in \mathcal{H}$,
$$
\mathbb{E} g (\chi^h_{T/h}) - \mathbb{E} g(X_T^h) = h^2 \sum_{k=0}^{\frac{T}{h} - 1} \mathbb{E} \tilde{\Phi}_{k}^h(\chi^h_k ) + h^2 \xi(h),
$$ 
where $\tilde{\Phi}^h(x) := \tilde{\varphi}^h (x).$
Then $\xi$ is bounded.
\end{lemma}
\begin{proof}
    As in Lemma \ref{taylorapplication}, Taylor expand $Y := \tilde{y}$ to yield
$$
\mathbb{E}Y_{k+1}^h (\chi_k^h) - \mathbb{e}Y_k^h (\chi_k^h) = h Z_1^h + h^2( Z_{xx}^h + Z_{tx}^h + Z_{tt}^h ) + \mathbb{E} R^h(\Delta \chi_k^h) ,
$$
where here
$$
Z_1^h := \mathbb{E} ( \partial_t Y_k^h (\chi_k^h) + h^{-1}\nabla Y_k^h (\chi_k^h)^T \Delta \chi_k^h + \frac{1}{2}hu_{kh}^2 \textnormal{tr}(\nabla^2 Y_k^h (\chi_k^h)\Sigma(\chi_k^h)))
$$ 
and 
$$
Z_{xx}^h := \frac{1}{2} u_{kh}^2 \mathbb{E} \textnormal{tr} ( \nabla^2 Y_k^h (\chi_k^h) \bar{H}(\chi_k^h)\bar{H}(\chi_k^h)^T).
$$
By the Feynman-Kac equation (\ref{feynmankacsde}), we have $Z_1 = 0$, and the remaining terms proceed as in the proof of Lemma \ref{taylorapplication}.
\end{proof}

\begin{proposition}\label{expdifforderswap}
Let $W_t$ be a vector $(W_t^{(1)},W_t^{(2)},...W_t^{(d)})$, of Wiener processes. Define $X(z)$ as the solution of the SDE
\begin{equation}\label{SDEgeneralform}
dX_t = \mu(t,W_t) dt +\sigma(t,X_t) dW_t
\end{equation}
with initial condition $X_0 = z$, where $z \in \mathbb{R}^d$ . Define $Y$ as the solution of
$$
dY_t = \nabla\mu (t,X_t) Y_t dt + \nabla \sigma (t,X_t) Y_t dW_t
$$
with initial condition $Y_0 = e_j$. Then
$$
\mathbb{E}[ \nabla  f(X_t(z))^T Y(z) ] = \frac{\partial}{\partial z_j} \mathbb{E} [ f(X_t(z)) ] .
$$
\end{proposition}
\begin{proof}
Fix some $j \in \{ 1,2,...,d\}$.
By Lemma \ref{itolemma},
\begin{align*}
\frac{\partial}{\partial z_j} f(t,X_t) &= \frac{\partial}{\partial z_j} f(0,X_0) + \int_0^t \frac{\partial}{\partial t} \frac{\partial}{\partial z_j} f(s,X_s) ds + \sum_{i=1}^d  \int_0^t\frac{\partial}{\partial x_i} \frac{\partial}{\partial z_j} f(s,X_s) \mu(s,X_s)ds \\
&\hspace{5mm} + \sum_{i=1}^d \int_0^t \frac{\partial}{\partial x_i} \frac{\partial}{\partial z_j} f(s,X_s) \sigma(s,X_s) dW_s \\
&\hspace{5mm} + \frac{1}{2} \sum_{i=1}^d \sum_{k=1}^d  \int_0^t \
\frac{\partial ^2}{\partial  x_i \partial x_k} \frac{\partial}{\partial z_j} f(s,X_s) d\biggl \langle  \int_0^t \sigma^{(i)}(s,X_s) d W_s , \int_0^t \sigma^{(k)}(s,X_s) d W_s \biggr\rangle_s 
\end{align*}
Thus,
\begin{align*}
    \frac{\partial}{\partial z_j} \mathbb{E} [ f(X_t) ]  &= \frac{\partial}{\partial z_j} \mathbb{E} \biggl[ f(0,X_0) + \int_0^t \frac{\partial}{\partial t} f(s,X_s) ds + \sum_{i=1}^d  \int_0^t\frac{\partial}{\partial x_i} f(s,X_s) \mu(s,X_s)ds \\
&\hspace{5mm} + \sum_{i=1}^d \int_0^t \frac{\partial}{\partial x_i} f(s,X_s) \sigma(s,X_s) dW_s \\
&\hspace{5mm} + \frac{1}{2} \sum_{i=1}^d \sum_{k=1}^d  \int_0^t \
\frac{\partial ^2}{\partial  x_i \partial x_k}f(s,X_s) d\biggl \langle  \int_0^t \sigma^{(i)}(s,X_s) dW_s , \int_0^t \sigma^{(k)}(s,X_s) dW_s \biggr \rangle_s \bigg] \\
&= \frac{\partial}{\partial z_j} f(0,X_0)  + \frac{\partial}{\partial z_j} \mathbb{E} \biggl[ \int_0^t \frac{\partial}{\partial t} f(s,X_s) ds + \sum_{i=1}^d  \int_0^t\frac{\partial}{\partial x_i} f(s,X_s) \mu(s,X_s)ds \\
&\hspace{5mm} + \frac{1}{2} \sum_{i=1}^d \sum_{k=1}^d  \int_0^t \
\frac{\partial ^2}{\partial  x_i \partial x_k}f(s,X_s) d\biggl \langle  \int_0^t \sigma^{(i)}(s,X_s) dW_s , \int_0^t \sigma^{(k)}(s,X_s) dW_s \biggr \rangle_s \biggr] \\
&= \frac{\partial}{\partial z_j} f(0,X_0)  + \frac{\partial}{\partial z_j} \mathbb{E} \biggl[ \int_0^t \frac{\partial}{\partial t} f(s,X_s) ds + \sum_{i=1}^d  \int_0^t\frac{\partial}{\partial x_i} f(s,X_s) \mu(s,X_s)ds \\
&\hspace{5mm} + \frac{1}{2} \sum_{i=1}^d \sum_{k=1}^d  \int_0^t \
\frac{\partial ^2}{\partial  x_i \partial x_k}f(s,X_s) \sigma^{(i)} (s,X_s)  \sigma^{(k)}(s,X_s) ds \biggr]
\end{align*}
Likewise, by Corollary \ref{itoibp},
\begin{align*}
    \mathbb{E}[ \nabla f(s,X_s)^T Y_t ] &= \mathbb{E} \biggl[ \int_0^t \nabla f(s,X_t) dY_s +  \nabla f(0,X_0) Y_0 + \int_0^t Y_s d \nabla f(s,X_s)  \\
    &\hspace{5mm}  + \biggl \langle \int_0^t \nabla \sigma (s,X_s) Y_s d W_s , \sum_{i=1}^d \int_0^t \frac{\partial}{\partial x_i} \nabla f(s,X_s) \sigma(s,X_s)  dW_s \biggr \rangle \biggr] \\
    &= \mathbb{E} \biggl[ \int_0^t \nabla f(s,X_t) dY_s + \nabla f(s,X_0) Y_0 + \int_0^t Y_s d \nabla f(s,X_s)  \\
    &\hspace{5mm}  +  \sum_{i=1}^d \biggl \langle \int_0^t \nabla \sigma (s,X_s) Y_s d W_s , \int_0^t \frac{\partial}{\partial x_i} \nabla f(s,X_s) \sigma(s,X_s)  dW_s \biggr \rangle \biggr] \\
    &= \mathbb{E} \biggl[ \int_0^t \nabla f(s,X_t) dY_s + \nabla  f(0,X_0) Y_0 + \int_0^t Y_s d \nabla f(s,X_s)  \\
    &\hspace{5mm}  +  \sum_{i=1}^d \int_0^t  \nabla \sigma (s,X_s) Y_s \frac{\partial}{\partial x_i} \nabla f(s,X_s) \sigma(s,X_s) ds \\
    &= \nabla  f(0,X_0) Y_0 + \mathbb{E} \biggl [ \int_0^t \nabla f(X_s) dY_s + \int_0^t  Y_s d \biggl( \frac{\partial}{\partial z_j} f(0,X_0) + \int_0^t \frac{\partial}{\partial t} \frac{\partial}{\partial z_j} f(s,X_s) ds  \\
&\hspace{5mm} + \sum_{i=1}^d  \int_0^t\frac{\partial}{\partial x_i} \frac{\partial}{\partial z_j} f(s,X_s) \mu(s,X_s)ds + \sum_{i=1}^d \int_0^t \frac{\partial}{\partial x_i} \frac{\partial}{\partial z_j} f(s,X_s) \sigma(s,X_s) dW_s \\
&\hspace{5mm} + \frac{1}{2} \sum_{i=1}^d \sum_{k=1}^d  \int_0^t \
\frac{\partial ^2}{\partial  x_i \partial x_k} \frac{\partial}{\partial z_j} f(s,X_s) d\biggl \langle  \int_0^t \sigma^{(i)}(s,X_s) ds , \int_0^t \sigma^{(k)}(s,X_s) ds \biggr\rangle_s \biggr) \\
&\hspace{5mm} + \sum_{i=1}^d \int_0^t  \nabla \sigma (s,X_s) Y_s \frac{\partial}{\partial x_i} \nabla f(s,X_s) \sigma(s,X_s) ds \biggr] \\
&= \nabla  f(0,X_0) Y_0 + \mathbb{E} \biggl [ \int_0^t \nabla f(s,X_s) \nabla \mu_t Y_t ds  + \int_0^t  Y_s d \biggl( \frac{\partial}{\partial z_j} f(0,X_0)   \\
&\hspace{5mm} + \int_0^t \frac{\partial}{\partial t} \frac{\partial}{\partial z_j} f(s,X_s) ds + \sum_{i=1}^d  \int_0^t\frac{\partial}{\partial x_i} \frac{\partial}{\partial z_j} f(s,X_s) \mu(s,X_s)ds \\
&\hspace{5mm} + \frac{1}{2} \sum_{i=1}^d \sum_{k=1}^d  \int_0^t \
\frac{\partial ^2}{\partial  x_i \partial x_k} \frac{\partial}{\partial z_j} f(s,X_s) d\biggl \langle  \int_0^t \sigma^{(i)}(s,X_s) ds , \int_0^t \sigma^{(k)}(s,X_s) ds \biggr\rangle_s \biggr) \\
&\hspace{5mm} + \sum_{i=1}^d \int_0^t  \nabla \sigma (s,X_s) Y_s \frac{\partial}{\partial x_i} \nabla f(s,X_s) \sigma(s,X_s) ds \biggr]
\end{align*}
\begin{align*}
    &= \frac{\partial}{\partial z_j }  f(0,X_0) + \mathbb{E} \biggl [ \int_0^t \nabla f(s,X_s) \nabla \mu_t Y_t ds  + \int_0^t  Y_s \frac{\partial}{\partial t} \frac{\partial}{\partial z_j} f(s,X_s) ds  \\
&\hspace{5mm}   + \sum_{i=1}^d  \int_0^t Y_s \frac{\partial}{\partial x_i} \frac{\partial}{\partial z_j} f(s,X_s) \mu(s,X_s)ds \\
&\hspace{5mm} + \frac{1}{2} \sum_{i=1}^d \sum_{k=1}^d  \int_0^t Y_s 
\frac{\partial ^2}{\partial  x_i \partial x_k} \frac{\partial}{\partial z_j} f(s,X_s) \sigma^{(i)}(s,X_s) \sigma^{(k)}(s,X_s) ds \\
&\hspace{5mm} + \sum_{i=1}^d \int_0^t  \nabla \sigma (s,X_s) Y_s \frac{\partial}{\partial x_i} \nabla f(s,X_s) \sigma(s,X_s) ds \biggr] \\
&= \frac{\partial}{\partial z_j }  f(0,X_0) + \frac{\partial}{\partial z_j} \mathbb{E} \biggl[ \int_0^t \frac{\partial}{\partial t} f(s,X_s) ds + \sum_{i=1}^d  \int_0^t\frac{\partial}{\partial x_i} f(s,X_s) \mu(s,X_s)ds \\
&\hspace{5mm} + \frac{1}{2} \sum_{i=1}^d \sum_{k=1}^d  \int_0^t 
\frac{\partial ^2}{\partial  x_i \partial x_k}f(s,X_s) \sigma^{(i)} (s,X_s)  \sigma^{(k)}(s,X_s) ds \biggr] \\
&= \frac{\partial}{\partial z_j} \mathbb{E} [ f(z,X_t) ] .
\end{align*}
\end{proof}
\begin{lemma}\label{itotaylorexpansion}
    Let $b^0,b^1,\sigma \in G_1([0,\infty) \times \mathbb{R}^d] \cap G^\infty ( [0,\infty) \times \mathbb{R}^d ) ) $, such that $b^0,b^1\in \mathbb{R}^d$ and $\sigma \in \mathbb{R}^{d \times d}$. Let $h \in (0,1)$, $n \in \mathbb{N}$ and consider the SDE
    $$
    dX_t = (b^0_t + hb^1_t)(X_t^h)dt + \sqrt{h}\sigma_t (X_t^h) dW_t, \quad X_{nh} = x,
    $$
    with $t \in [nh,(n+1)h]$. Then there exists a function $C \in G(\mathbb{R}^d)$ such that
    \begin{align}
        \mathbb{E} \Delta \tilde{X}_n^{h,n} &= hb_{n,h}^0 + \frac{1}{2}h^2 (2b_{n,h}^1 + (\nabla b^0b^0)_{nh} + \dot{b}_{nh}^0) + h^3 C ,\notag \\
        \mathbb{E}(\Delta \tilde{X}_n^{h,n})(\Delta \tilde{X}_n^{h,n})^T &= h^2 ((b^0)(b^0)^T + \sigma \sigma^T)_{nh} + h^3 C.
    \end{align}
for all $h \in (0,1)$.
\end{lemma}
\begin{proof}
    See Lemma 5.1, \cite{ankirchner2021stochasticapproximation}
\end{proof}
\begin{lemma}\label{ding2}
Let $y^h_t(x) = \mathbb{E}g(X_T^{h,t}(x))$ and
$$
d_t^h(x) := \frac{y_t^h(x) -y^0_t(x)}{h}.
$$
Then $d \in G^2 ( [0,T] \times \mathbb{R}^d ) $.
\end{lemma}
\begin{proof}
    Using Lemma \ref{itotaylorexpansion}, we have via \cite{ankirchner2021stochasticapproximation}, Lemma 4.4, that for some $C\in G(\mathbb{R}^d)$, all $h \in \mathcal{H}$ , $t \in [0,T]$,
    $$
    \abs{y_t^h- y_t^0} \leq Ch  .
    $$
Now let $Y_r$ satisfy
$$
dY_r = u_r \nabla \bar{H}(X_r^{h,t}(x))Y_r \, dr + u_r \sqrt{h} \nabla \sqrt{\Sigma (X_r^{h,t}(x))} Y_r dW_r 
$$
with initial condition $Y_t = z$, where $z \in \mathbb{R}^d$, and define
$$
w_t^h:= \mathbb{E}[ \nabla g(X_T^{h,t} (x))^T \partial_j Y_T^{h,t} ] \text{.}
$$
Note that, by Proposition \ref{expdifforderswap}, $w^h(x,e_j) = \partial_j y^h (x)$. By Lemma \ref{feynmankacv2}, we have
\begin{align*}
    -\partial_t w_y^h(x,z) - u_t\nabla_x w_t^h(x,z) \bar{H}(x) &= \nabla_y w_t^h (x,z) y  \partial_j \bar{H}(x) + \frac{1}{2} h u_t^2 \textnormal{tr} ( 
 \nabla_{x,y}^2w_t^h (x,z) S(x,z) ) ,
\end{align*}
where 
$$
S(x,z) = \begin{pmatrix}
  \Sigma(x) & \sqrt{\Sigma(x)}( \nabla \sqrt{\Sigma(x)}  z )^T\\ 
  \nabla\sqrt{\Sigma(x)}  z
 \sqrt{\Sigma(x)}^T & ( \nabla  \sqrt{\Sigma(x)}  z)
 (\nabla \sqrt{\Sigma(x)} z)^T\\
\end{pmatrix}
$$
Following a Taylor expansion as in, we have that
\begin{align*}
x \mapsto \frac{1}{h} [ \mathbb{E} w^0_{t + (n+1)h}( X^h_{t + (n+1)h} (x) , \partial_j X^h_{t + (n+1)h}(x,1) ) \\
-  \mathbb{E} w^0_{t + nh}( X^h_{t + nh} (x) , \partial_j X^h_{t + nh}(x,1) ) ] \in G(\mathbb{R}^d)
\end{align*}
$$
\frac{1}{h} \abs{\partial_j y_t^h - \partial_j y_t^0} \in G(\mathbb{R}^d),
$$
and conclude via the Feynman-Kac equation for $y$ that
$$
\frac{1}{h} \abs{\partial_t y_t^h - \partial_t y_t^0} \in G(\mathbb{R}^d).
$$
\end{proof}

\subsection{Proof of Theorem \ref{momentumsderesult}}
As with the ODE case, we begin by fixing $g \in G^\infty (\mathbb{R}^d)$ and letting $y_t^h (x) = \mathbb{E}g(X_T^{h,t}(x))$. By Lemma \ref{taylorapplicationsde}, we have
$$
\mathbb{E}g(\chi_{T/h}^h) - \mathbb{E}g(X_T^h) = h^2 \sum_{n=0}^{\frac{T}{h} - 1} \mathbb{E} \tilde{\Phi}_{n}^h (\chi_n^h) + \mathcal{O} (h^2).
$$
Furthermore, there exists a constant $C > 0$ such that
$$
\sum_{n=0}^{\frac{T}{h} - 1} \abs{\mathbb{E}\Phi_{nh}(\chi_n^h) - \Phi_{nh}(X_nh)} \leq C
$$
for all  $h \in \mathcal{H}$, thus we arrive at the bound analogous to the ODE case,
\begin{equation}\label{prelimboundsde}
\mathbb{E}g(\chi_{T/h}^h) - \mathbb{E}g(X_T^h) = h^2 \int_0^T \mathbb{E} \varphi_{n}^h (\chi_n^h) + \mathcal{O} (h^2).
\end{equation}
Now, by Lemma \ref{ding2} we have
$$
\abs{\mathbb{E}\varphi^h_t (X_t^h) -  \mathbb{E}\varphi_t^0 (X_t^0)} \in \mathcal{O}(h) ,
$$
which in conjunction with Theorem \ref{basesdebound} gives us
$$
\abs{\varphi_t^h (x) - \varphi_t^0 (x)} \leq h C_1 (x) 
$$
for some $\kappa \in \mathbb{N}$, $C_1 \in G_\kappa (\mathbb{R}^d)$. Then, by Lemma \ref{itotaylorexpansion}, we have
$$
\abs{\mathbb{E}\varphi_t^0 (X_t^h) - \varphi_t^0 (X_t^0)} \in \mathcal{O}(h) .
$$
The above results allow us to conclude
$$ \abs{\mathbb{E}\varphi_t^h (X_t^h) - \varphi_t^0 (X_t^h)} \in \mathcal{O}(h),$$ therefore we can extend (\ref{prelimboundsde}) to the desired equality
$$
\mathbb{E}g(\chi_{T/h}^h) - \mathbb{E} g(X_T^h) = h \int_0^T \varphi_t^0 (X_t^0) dt + \mathcal{O}(h^2).
$$
\subsection{Proof of Theorem \ref{sgdmomentumapproxmain}, SDE Case}
For $n \geq 2$, let $\tilde{\chi}_n$ be the process described by (\ref{modifiedlrsgd}) with constant learning rate and momentum parameter $\mathring{\eta}$ and $\mathring{\zeta}$, respectively. Suppose we are given initial conditions $x_1, x_0$. We rewrite this as a $2d$-dimensional process $\chi_n$,
where
\begin{equation}
    \chi_{n+1}^h = \chi_n^h + \eta_n^h J_{\gamma(n)}(\chi^h_n),
\end{equation}
with
$J_{\gamma(n)}(\chi_n) = \nabla j_{\gamma(n)} (\chi_n)$, where given $x \in \mathbb{R}^{2d}$,
$$
j(x) = f(x_1,...,x_d) + \sum_{i=1}^d \Bigl[ \eta^{-1}(1+\zeta)\frac{x_i^2}{2} + \zeta \eta^{-1}\frac{x_{d+i}^2}{2} - \eta^{-1}\zeta x_i  x_{d+i}\Bigr]
$$
and
$$
\eta^h = \begin{pmatrix}
\mathring{\eta} & & & & & \\
& \ddots & & & &\\
& & \mathring{\eta} & & & \\
 & & & -\mathring{\eta}\mathring{\zeta}^{-1} & &  \\
 & & & & \ddots & \\
 & & & & & -\mathring{\eta}\mathring{\zeta}^{-1}
\end{pmatrix}.
$$
Then, apply Theorem \ref{momentumsderesult}.
\section{Second Order SDE Approximation}\label{section2ndorder}
The techniques in the previous section may also be extended to higher order results, the statements of which may take the form of the following theorem.
\begin{theorem}
    Assume (A1),(A2), and (A3).  For all $h \in (0,1)$ let $X^h$ be the solution of
    $$
    d X_t^h = \biggl( u_t \bar{J}(X_t^h) - \frac{1}{2} h (y_t^2 \nabla \bar{J}\bar{J} + \dot{u}_t \bar{J})(X_t^h) \biggr) dt + u_t \sqrt{h \Sigma (X_t^h)} dW_t 
    $$
with initial condition $X_0^h = (x_1,x_0)$. Then for all $g \in G^\infty (\mathbb{R}^d)$ and $T > 0$,
$$
\max_{n \in \{ 0,...,\lfloor T/h \rfloor \}} \abs{\mathbb{E}g(\chi_n^h) - \mathbb{E} g(X_{nh}^h)} \in \mathcal{O}(h^2) 
$$
as $h \rightarrow 0$.
\end{theorem}
We begin by giving a proof of the following lemma.
\begin{lemma}\label{sdefunctiondiffbound}
Let $b,\sigma$, and $X$ be defined as in Lemma \ref{sdedifferencebound}. Now let $\ell \in \mathbb{N}$ and $k \in \mathbb{N}_0$. Suppose we are given a function
$$
f: [0,T] \times (0,1) \times \mathbb{N} \times \mathbb{R}^d  \rightarrow \mathbb{R}
$$
with $f_k^{i,h}(x) \leq C_f(1 + \abs{x}^{k+1})$ for all $x \in \mathbb{R}^d$ and that there exists a function $M \in G_\kappa(\mathbb{R}^d)$ for some $\kappa \in \mathbb{N}$ such that
\begin{align}
    \abs{\mathbb{E}(\Delta \chi_k^{h,k})^\alpha - \mathbb{E}(\Delta \tilde{E}_k^{h,k})^\alpha} &\leq h^{\ell + 1}M \quad \abs{\alpha} \leq \ell \label{taylorboundassumption1} \\
    \| \Delta \chi_k^{h,k} \|_p^{\ell + 1} , \| \Delta \tilde{X}_k^{h,k}\|_p^{\ell + 1}  &\leq h^{\ell + 1} M , \quad p \in  \{ 2, 2 \ell + 2 \} \label{taylorboundassumption2}
\end{align}
for all $h \in (0,1)$ and $k \in \{ 0,..., \lfloor T/h \rfloor \}$.
Then there exists a function $C\in G_\kappa(\mathbb{R}^d)$ such that for all $h \in (0,1)$, $t \in [0,T]$, $k \in \{ 0,..., \lfloor T/h \rfloor \}$,
$$
\abs{ \mathbb{E} f_k^{t,h}(\chi_{k+1}^{h,k}) - \mathbb{E}f_k^{t,h}(\tilde{X}^{h,k}_{k+1}) }  \leq h^{\ell + 1}C.
$$
\end{lemma}
\begin{proof}
Firstly, note that
\begin{align*}
     f_k^{i,h}(\chi_{k+1}^{h,k}) - f_k^{i,h}(\tilde{X}^{h,k}_{k+1}) &= f_k^{i,h}(\chi_{k+1}^{h,k}) - f_k - ( f_k^{i,h}(\tilde{X}^{h,k}_{k+1}) - f_k ) .
\end{align*}
Then for any $x \in \mathbb{R}^d$, we have by Taylor's theorem, for every $h \in (0,1)$ and $k \in \{ 0,..., \lfloor T/h \rfloor \} $, there exist $\theta_{\Delta \chi_{k}^{h,k}},\theta_{\Delta \tilde{X}^{h,k}_{k}} \in (0,1)$ such that
\begin{align*}
     f_k^{i,h}(\chi_{k+1}^{h,k}) - f_k^{i,h}(\tilde{X}^{h,k}_{k+1}) &= \sum_{0 < \abs{\alpha} \leq \ell} \frac{1}{\alpha !} \partial_\alpha f_k \cdot ((\Delta\chi_k^{h,k})^\alpha - (\Delta \tilde{X}_k^{k,h})^\alpha) \\
     &\hspace{5mm} + \sum_{\abs{\beta} = \ell+ 1} \frac{1}{\beta !} \partial_\beta f_k ( x + \theta_{\Delta \chi}\Delta \chi ) \Delta \chi(x)^\beta \\
     &\hspace{5mm} + \sum_{\abs{\beta} = \ell+ 1} \frac{1}{\beta !} \partial_\beta f_k ( x + \theta_{\Delta \tilde{X}} \Delta \tilde{X} ) \Delta \tilde{X}(x)^\beta
\end{align*}
Since $f \in G_{\ell + 1}(\mathbb{R}^d)$, there exist constants $c_1 > 0$, $\lambda \in \mathbb{N}$ such that 
\begin{align*}
    \abs{\mathbb{E}[\partial_\beta f(x + \theta_{\Delta \tilde{X}^h}\Delta \tilde{X}^h(x))\Delta \tilde{X}^h(x)^\beta]} &\leq \| \partial_\beta f \|_{G_\lambda} \Bigl(1 + \abs{ x + \Delta \tilde{X}(x) }^\lambda \Bigr) \| \Delta \tilde{X}(x) \|^{\ell + 1}_{2\ell + 2} \\
    &\leq \| \partial_\beta f \|_{G_\lambda} \Bigl(1 + 2^{\lambda - 1}\abs{x}^\lambda \\
    &\hspace{5mm} + 2^{\lambda - 1} \| \Delta \tilde{X}(x) \|^\lambda_2 \Bigr)  \| \Delta \tilde{X}(x)\|_{2\ell + 2}^{\ell + 1}. \quad (\dagger)
\end{align*}
Furthermore, we have by (\ref{taylorboundassumption2})
\begin{align*}
(\dagger)   &\leq  c_1(1 + \abs{x}^\lambda + M )h^{\ell + 1} M .
\end{align*}
Similarly, for $\Delta \chi$, we have a constant $c_2 > 0$ for which
\begin{align*}
    \abs{\mathbb{E}[\partial_\beta f(x + \theta_{\Delta \chi^h}\Delta \chi^h(x))\Delta \chi^h(x)^\beta]} &\leq  c_2(1 + \abs{x}^\lambda + M )h^{\ell + 1} M
\end{align*}
Therefore, for some constant $c > 0$,
\begin{align*}
    \abs{ \mathbb{E} f_k^{i,h}(\chi_{k+1}^{h,k}) - \mathbb{E}f_k^{i,h}(\tilde{X}^{h,k}_{k+1}) } &\leq c \sum_{0 < \abs{\alpha} \leq \ell} \| \partial_\alpha f \|_{G_\kappa} (1 + \abs{x}^\kappa)h^{\ell+1}M &&\text{by (\ref{taylorboundassumption1})} \\
    &\hspace{5mm} + M \sum_{\abs{\beta} = \ell + 1} \| \partial_\beta f \|_{G_\kappa} ( 1 + \abs{x}^\kappa + M )h^{\ell+1}M . \\
    &\leq h^{\ell + 1} C,
\end{align*}
where $C \in G_\kappa(\mathbb{R}^d)$.
\end{proof}
\noindent The remainder of the second-order approximation proof follows largely the same as in \cite{ankirchner2021stochasticapproximation}.
\section{Conclusion}\label{sectionconclusion}
While benefits of momentum-based SGD methods have been observed in several other works, there is a lack of mathematical proof regarding its effectiveness and optimal scenarios. With the approximation results given in this paper, the dynamics of these methods can be better understood through stochastic calculus. The generalization of learning rate schedules opens up a wide variety of applications and extension to other methods such as Nesterov Acceleration.  
\bibliographystyle{plainnat}

\end{document}